\pdfoutput=1

\documentclass[11pt]{article}

\usepackage[final]{acl}

\usepackage{times}
\usepackage{latexsym}
\usepackage{microtype}
\usepackage{graphicx}
\usepackage{subfigure}
\usepackage{booktabs} 
\usepackage{latexsym}
\usepackage{tabularx} 
\usepackage{array}
\usepackage{multirow}
\usepackage{hyperref}
\usepackage{amsmath}
\usepackage{amssymb}
\usepackage{mathtools}
\usepackage{amsthm}
\usepackage{amsmath}
\usepackage{amsthm}
\usepackage{amssymb}

\usepackage[textsize=tiny]{todonotes}
\usepackage[capitalize,noabbrev]{cleveref}

\theoremstyle{plain}
\newtheorem{theorem}{Theorem}[section]

\newtheorem{lemma}[theorem]{Lemma}

\theoremstyle{definition}

\theoremstyle{remark}

\usepackage[T1]{fontenc}

\usepackage[utf8]{inputenc}

\usepackage{microtype}

\usepackage{inconsolata}

\usepackage{graphicx}

\title{ThoughtProbe: Classifier-Guided LLM Thought Space Exploration\\ via Probing Representations}

\author{Zijian Wang \\
  School of Computer Science \\
  The University of Sydney \\
  \texttt{zwan0998@uni.sydney.edu.au} \\\And
 Chang Xu\\
 School of Computer Science \\
 The University of Sydney \\
  \texttt{c.xu@sydney.edu.au} \\}

\begin{document}
\maketitle
\begin{abstract}
This paper introduces ThoughtProbe, a novel inference-time framework that leverages the hidden reasoning features of Large Language Models (LLMs) to improve their reasoning performance.
Unlike previous works that manipulate the hidden representations to steer LLM generation, we harness them as discriminative signals to guide the tree-structured response space exploration.
In each node expansion, a classifier serves as a scoring and ranking mechanism that efficiently allocates computational resources
by prioritizing higher score candidates for continuation.
After completing the tree expansion, we collect answers from all branches to form a candidate answer pool. 
We then propose a branch-aggregation method that marginalizes over all supporting branches by aggregating their CoT scores, thereby identifying the optimal answer from the pool.
Experimental results show that our framework's comprehensive exploration not only covers valid reasoning chains but also effectively identifies them, achieving significant improvements across multiple arithmetic reasoning benchmarks. The code is available at \url{https://github.com/Zijian007/Thoughtprobe}.
\end{abstract}

\section{Introduction}


Chain of Thought (CoT) reasoning has emerged as a pivotal approach for enhancing LLMs' problem-solving capabilities\cite{wei2022chain}.
However, eliciting this capability from pre-trained base LLMs typically requires expensive post-training or carefully designed prompting strategies\cite{yao2023tree, kojima2022large, hoffman2024training}.

Recent research demonstrates that LLMs' internal hidden representations serve as meaningful proxies for CoT behaviors, revealing a correspondence between reasoning patterns and specific linear features within the internal activation space\cite{ye2024physics}. 
This correspondence has given rise to two distinct insights for leveraging  representations to improve reasoning performance.
The first insight adopts a causality perspective, viewing hidden representations as causal factors that influence the CoT generation.
This has led to activation steering techniques that manipulate representations along specific directions to enhance reasoning capabilities\cite{hong2025reasoning, tang2025unlocking, hojerimproving}. 

Despite promising results, such approaches face inherent limitations. 
Direct manipulation risks disrupting the model's internal representational structure, potentially pushing activations out of distribution and degrading linguistic quality\cite{von2024language, da2025steering}. 
Moreover, the high-dimensional nature of the latent space makes it challenging for a single linear direction to capture the complexity of reasoning features, which often involve intricate patterns spanning multiple cognitive dimensions\cite{luo2024pace, bo2025steerable}.

\begin{figure*}[t]
  \begin{center}
  \centerline{\includegraphics[scale = 0.46]{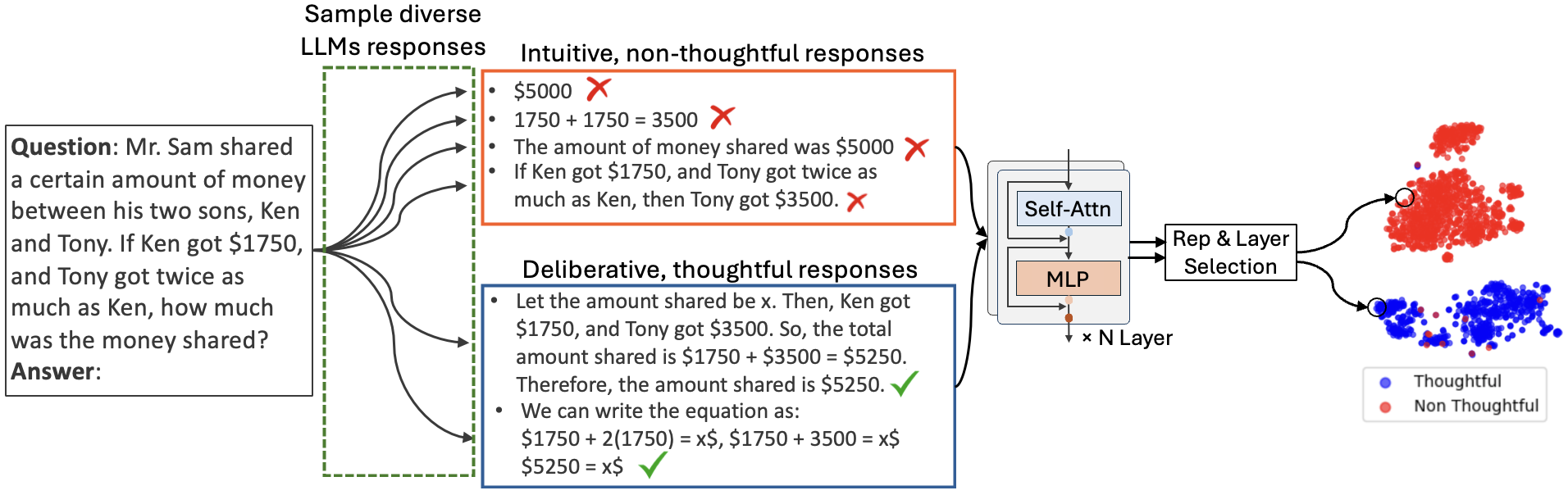}}
  \caption{Pre-trained LLMs could naturally generate both CoT and non-CoT responses when sampling multiple times, and hidden representations provide a strong signal for discriminating them.}
  \label{fig:intro}
  \end{center}
  \vspace{-2mm}
\end{figure*}

In this work, we adopt an alternative perspective, which recognizes the strong correlation between hidden representations and the manifestation of CoT in generated text.
Rather than manipulating representations to steer LLM generation, we leverage their discriminative capacity as indicators to detect reasoning patterns within the model's natural outputs.
Through rigorous empirical investigation, we first demonstrate that representations exhibit remarkable power in distinguishing between CoT and non-CoT content, particularly within specific representation types and network layers, as evidenced by a simple classifier's performance. 
Also, we show the classifier is a reliable evaluator that can assign higher scores to high quality CoT content, supported by both theoretical evidence and empirical validation.

Building on these findings, we present ThoughtProbe, a novel inference-time computational framework that effectively explores CoT paths via response space exploration.
Specifically, ThoughtProbe systematically explores the response space as an iterative tree expansion process, with the input question as the root node, and branches as the candidate CoT paths.
At each expansion step, multiple token sequences are generated in parallel as candidate child nodes, whose CoT score is evaluated by the classifier through probing their hidden representations. 
By prioritizing higher scoring candidates for continuation, we efficiently allocate computational resources and increase the likelihood of including correct reasoning paths in our exploration tree. 
This exploration process continues until either reaching the termination token or exhausting the computational budget.

Upon completion of tree expansion, we obtain multiple branches, each leading to a candidate answer, forming a comprehensive answer pool. 
To determine the optimal one, rather than using Best-of-N sampling\cite{huang2025best,sun2024fast}, we propose a branch-aggregation selection through value marginalization that considers CoT score across all branches leading to each candidate answer. 
Specifically, the value of each answer is computed by aggregating the CoT score of all its supporting branches, with the final answer selected as the one that achieves the highest marginal value.

Experiments on multiple reasoning benchmarks demonstrate that ThoughtProbe consistently outperforms existing inference-time computing methods, achieving significant improvements over both sampling-based methods (\textit{e.g.,} self-consistency) prompting-based techniques (\textit{e.g.,} zero-shot CoT and ToT) and activation-steering method. 
Our work provides new insights into enhancing LLMs' reasoning capabilities without requiring expensive fine-tuning or elaborate prompting strategies, and opens up promising directions for developing more robust reasoning systems that can effectively leverage the model's internal representations.

\section{Preliminary}
\subsection{LLMs Architecture and Hidden Representation}

To provide a foundation for the discussion, we first describe the basic structure of a Transformer-based LLM architecture \cite{vaswani2017attention}. 
The input text is initially tokenized into a sequence of tokens, which are then mapped to embeddings to form the initial representation sequence $ \mathbf{x}^{(0)} \in \mathbb{R}^{T \times d_\text{emb}}$.
Here, $T$ is the sequence length, and $d_\text{emb}$ is the embedding dimension.

The embeddings are then processed through multiple Transformer layers.
In each layer $l$, its hidden representations are composed of three components: activations from multi-head self-attention (MHA), multi-layer perceptron (MLP), and residual connections. This process can be formulated as:

\vspace{-5mm}
\begin{align*}
\mathbf{a}^{(l)}_\text{attn} &= \text{MHA}(\mathbf{h}^{(l)}) & \text{(Att activations)} \\
\mathbf{a}^{(l)}_\text{mlp} &= \text{MLP}(\mathbf{a}^{(l)}_\text{attn} + \mathbf{h}^{(l)})  & \text{(MLP activations)} \\
\mathbf{h}^{(l+1)} &= \mathbf{a}^{(l)}_\text{mlp} + \mathbf{a}^{(l)}_\text{attn} + \mathbf{h}^{(l)}  & \text{(Hidden states)}
\end{align*}

\subsection{LLMs Reasoning Structure}
Reasoning structures typically manifest in two fundamental topologies: sequential chains and branching trees. 
The chain structure reflects the step-by-step nature of logical deduction, while the tree structure captures the exploration of multiple potential reasoning paths. 
Below, we formally define these structures and their probabilistic formulations.

\textbf{Reasoning Chain:} For an input question $Q$, a reasoning chain is defined as a sequence of intermediate thought steps $R = [Q, r_1, r_2, ..., r_N]$, leading to a final answer $A$. 
Here, $r_i$ represents an intermediate thought at the $i$-th step, and $N$ denotes the chain length. 
The answer can be extracted by appending a trigger prompt at the end of the chain, like ``Therefore, the answer is".
The probability of generating such a chain can be formalized as:
\vspace{-2mm}
\begin{align*}
   P(R, A|Q) &= P(r_1|Q) \prod_{i=2}^N P(r_i|Q, r_{1:i-1}) \nonumber \\
   &\cdot P(A|Q, R)
\end{align*}
where $P(r_1|Q)$ is the probability of the first step, $P(r_i|Q, r_{1:i-1})$ is the probability of the $i$-th step, and $P(A|Q, R)$ is the probability of the final answer. 
At each step $i$, a new thought $r_i$ is appended to form $R = [Q, r_1, ..., r_{i-1}, r_i]$.

More specifically, each reasoning step $r_i$ is itself a token sequence, which can be further decomposed as:
\vspace{-2mm}
\begin{equation*}
P(r_i|Q, r_{1:i-1}) = \prod_{t=1}^{T_i} P(r_i^t|Q, r_{1:i-1}, r_i^{1:t-1})
\end{equation*}
where, $r_i^t$ denotes the $t$-th token in the $i$-th reasoning step, $T_i$ represents the total number of tokens in the $i$-th step, $r_i^{1:t-1}$ represents the previously generated tokens in the current step. 
In each token generation, the hidden representation $Rep(r_i^t)$ of token $r_i^t$ is accessible for probing.

\textbf{Branching Chains into Trees:} 
By sampling diverse tokens at each reasoning step, a single chain can branch into a tree structure, where $Q$ serves as the root node, each node represents an intermediate reasoning step.
This tree-based expansion explores multiple reasoning branches simultaneously and can increase the probability of covering the correct reasoning chain and answer.
At each step $r_i$, we could sample $k$ different continuations:
\vspace{-1mm}
\begin{equation*}
\{r_i^1, r_i^2, ..., r_i^k\} \sim P_k(r_i \mid Q, r_{1:i-1})
\end{equation*}

Here, $r_i^j$ represents the $j$-th sampled continuation at step $i$. 
Each root-to-leaf chain forms a distinct branch, leading to its answer, and collectively these branches generate an answer pool $\mathcal{A} = \{A_1, A_2, ..., A_p\}$.

While the tree structure improves solution coverage, it introduces two key challenges: 
(1) Candidate Selection: How to evaluate and prioritize promising children nodes in each exploration step? 
(2) Answer Determination: How to select the optimal answer from the pool $\mathcal{A}$?

\section{ThoughtProbe: Classifier-guided Reasoning Tree Exploration}
This section presents our ThoughtProbe framework that guide the response space exploration where the guidance signal is derived by probing representations.
We first validate the discriminative power of representations in discriminating CoT and non-CoT responses through comprehensive probing experiments across different LLMs.
We then introduce a classifier-guided beam search algorithm that systematically explores the response space to construct a diverse answer pool. 
Finally, we propose marginalization methods to aggregate these answers based on CoT score, enabling effective optimal answer selection.

\subsection{Probing Representations}
\label{main:Linear Probing in CoT}

\textbf{Setup}
We construct a binary representation classification dataset by first collecting paired CoT/non-CoT responses for questions sampled from GSM8K \cite{cobbe2021training} training set. 
For each question, we generate 10 distinct responses and classify them using GPT4o-as-Judge. 
We define CoT responses as those exhibiting correct step-by-step reasoning processes, while non-CoT responses provide answers directly without intermediate reasoning steps. 
Subsequently, we extract token-level representations from three widely-used LLMs: Mistral-7b \cite{jiang2023mistral}, Gemma-2-2b \cite{team2024gemma}, and Phi-1.5 \cite{li2023textbooks}, capturing activations across various layers and representation types.
More details are provided in the appendix\ref{app: classifier training data}.

\textbf{Classifier}
We employ Logistic Regression (LR) as our classifier. 
LR models the probability of CoT through a two-step process: first computing the logit (log-odds) using a linear function $\mathbf{w}^\top \mathbf{x} + b$, then transforming it to probability of positive via the sigmoid function $\sigma$.
\vspace{-2mm}
\begin{align*}
\text{logit} = \ln\frac{P(y=1|\mathbf{x})}{P(y=0|\mathbf{x})} &= \mathbf{w}^\top \mathbf{x} + b \\
P(y=1 | \mathbf{x}) = \sigma(\text{logit}) &= \frac{1}{1 + e^{-(\mathbf{w}^\top \mathbf{x} + b)}}
\end{align*}
where $\mathbf{w}$ is the weight vector, $b$ is the bias term, and $\mathbf{x}$ is the input feature vector. 

For each layer and representation type (Hidden states, Attention activations, and MLP activations), we train LR classifiers and evaluate their performance using AUC-ROC, and F1-score.

\begin{figure}[ht]
\begin{center}
\centerline{\includegraphics[scale = 0.35]{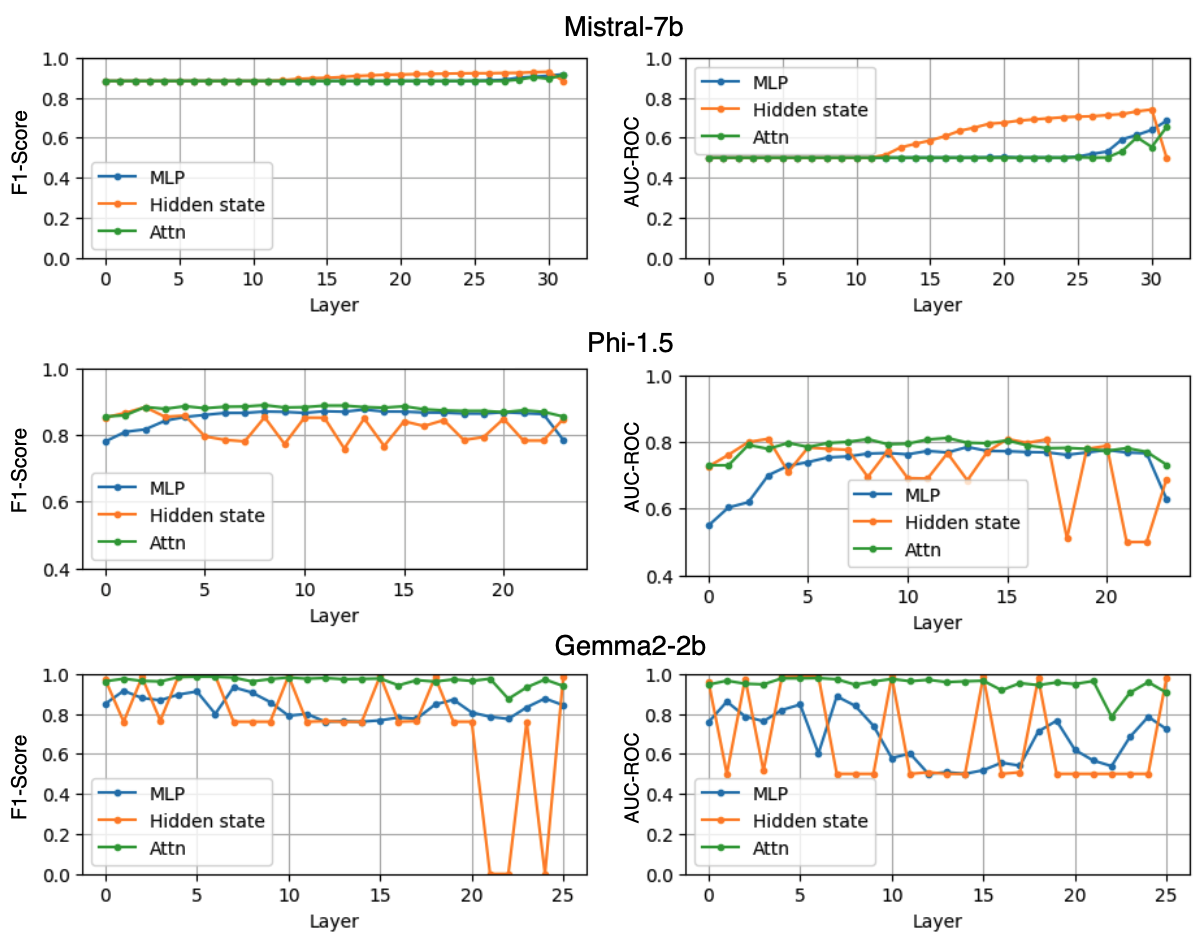}}
\caption{Layer-wise classification performance (F1-Score and AUC-ROC) across different representation types and LLMs.}
\label{classifier}
\end{center}
\vspace{-5mm}
\end{figure}

\textbf{Classification Results} Figure \ref{classifier} illustrates classification performance, that varies across representation types and layers in different LLMs.
(1) \textit{Representation type analysis:} In Mistral-7b, hidden states outperform both MLP and Attention activations. For Phi-1.5, Attention outputs demonstrate stable superiority despite hidden states' fluctuations. In Gemma2-2b, Attention outputs maintain consistent performance while hidden states and MLP activations fluctuate significantly. 
(2) \textit{Layer-wise analysis:} Layer depth influences performance differently across models. 
Mistral-7b shows a clear shallow-to-deep improvement trend, indicating progressive CoT feature refinement. 
Conversely, Phi-1.5 and Gemma2-2b exhibit fluctuating patterns with no consistent directional trends, suggesting more distributed CoT representations throughout layers.
Despite variations, we conclude that all LLMs achieve over 80\% performance with their optimal configurations, indicating the promising discriminative power of representations.

\textbf{Logit as Ranking Score}
Beyond the promising classification performance, we also validate that the classifier's logit can serve as a theoretically sound score for ranking and selecting higher CoT score candidates. 
Prior research\cite{sun2024rethinking} has demonstrated that a binary classifier's logit implies ordering equivalence with preference rewards in the Bradley-Terry model\cite{bradley1952rank}, establishing that:
\vspace{-1mm}
\begin{equation*}
l(x_1) > l(x_2) \implies r(x_1) > r(x_2)
\label{logit_order}
\end{equation*}
where $l(x)$ represents the logit value and $r(x)$ denotes the reward function in the Bradley-Terry model. 
A brief proof is provided in the appendix \ref{sec:proof}.
We empirically validate this ranking capability in Figure\ref{fig:score}. The left subplot shows CoT responses consistently achieve higher logit values than non-CoT responses, while the right subplot demonstrates correct CoT responses maintain higher logit values than incorrect ones. This suggests our classifier captures response quality regardless of reasoning correctness. Both theoretical and empirical evidence support using the classifier's logit as a ranking score for tree exploration.




\begin{figure}[h]
  \begin{center}
  \centerline{\includegraphics[scale = 0.30]{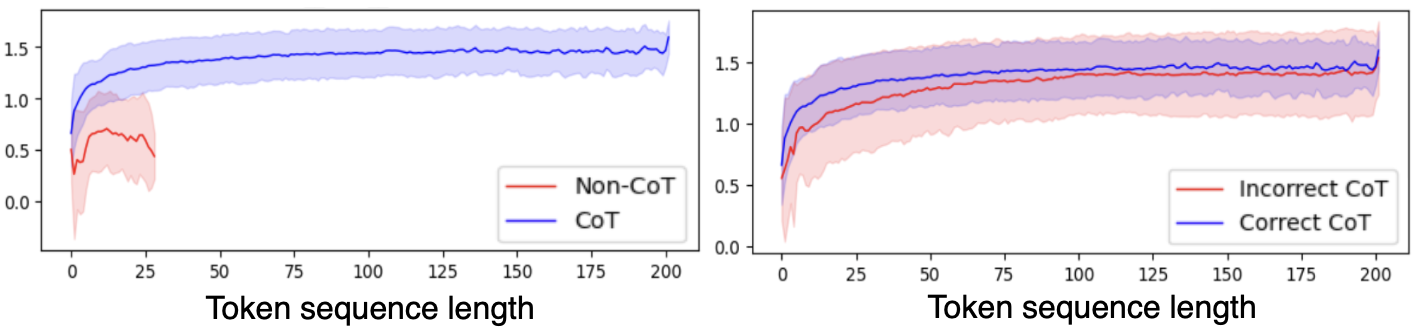}}
  \caption{Mean logit values and variance regions along the token sequence. 
  Left: Comparison between CoT and non-CoT responses. 
  Right: Comparison between correct and incorrect CoT responses.}
  \label{fig:score}
  \end{center}
  \vspace{-5mm}
  \end{figure}


\subsection{Classifier-guided Beam Search}
\begin{figure*}[ht]
    \begin{center}
    \centerline{\includegraphics[scale = 0.44]{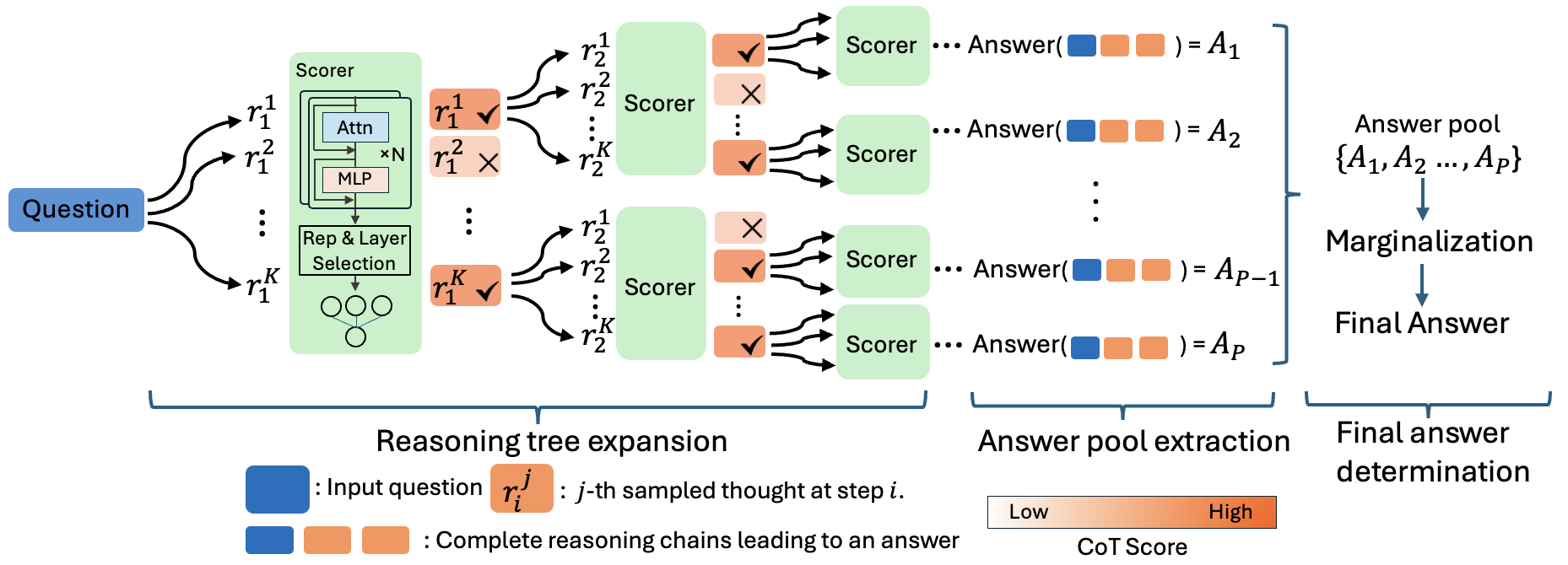}}
    \caption{Our classifier-guided tree exploration framework. 
    At each parent node, multiple candidates are sampled and evaluated by a pre-trained classifier by probing representations. 
    Nodes are selected for further expansion based on scores.
    Each exploration branch produces a candidate answer, forming an answer pool from which the final answer is determined through marginalization across all branches.}
    \vspace{-3mm}
    \label{framework}
    \end{center}
    \end{figure*}

With the classifier's logit as the ranking score, we propose a classifier-guided beam search for effective response space exploration.
Specifically, for a parent node (root question or intermediate reasoning step), the tree expansion process is formulated as follows:

\textbf{Diverse Beam Construction}: 
The process begins by generating diverse candidate continuations, organized into a beam. 
To encourage diversity, stochasticity must be introduced during token sequence generation.
In this paper, we employ Top-K-Start Greedy Decoding, which explores alternative top-$k$ tokens at the first decoding step, followed by greedy decoding for subsequent steps \cite{wang2024chain}.
The resulting $k$ reasoning chains, denoted as $B = \{R_1, R_2, ..., R_k\}$, represent potential continuations with associated hidden states, forming the initial beam for further processing.

\textbf{Derive CoT Score via Classifier}: 
Once the beam is constructed, a pre-trained classifier is used to evaluate the CoT score of each candidate . 
The classifier operates on the hidden state representations of the chains and assigns a score to each one. 
Specifically, for a candidate chain $R_i$, the CoT score $S_i$ is computed as $S_i = l(Rep(R_i[-1]))$, where $l(\cdot)$ is the logit output of classifier and $Rep(R_i[-1])$ represents the hidden state of the last token in $R_i$. 

\textbf{Beam Pruning by Score Ranking}: 
After scoring, all candidate chains are ranked based on their CoT scores, and only the top-$n$ highest-scoring candidates are retained for further expansion.
The pruned beam, denoted as $B'$, is defined as $B' = \{R_{\sigma(i)} \mid i \leq n\}$, where $\sigma$ is the permutation that sorts the scores in descending order, and $R_{\sigma(i)}$ represents the candidate corresponding to the $i$-th highest score. 
By dynamically adjusting the beam width $n$, we can control the trade-off between exploration breadth and computational efficiency. 
This pruning step ensures that only promising reasoning paths are preserved, effectively reducing computational overhead while maintaining the quality of the reasoning process. 

\textbf{Implementation Details}: 
Our framework consists of two phases: a branching phase for systematic exploration with depth $m$ and beam width $n$, followed by a completion phase for final generation.
During the branching phase, we iteratively expand the tree for $m$ steps. 
At each step $i$, we first generate $k$ candidate responses for each node and select the top-$n$ candidates based on their CoT scores, with each candidate expanded by generating a sequence of $T_i$ tokens.
In the completion phase, all leaf nodes from the branching phase are extended using greedy decoding until either reaching a completion token or the maximum length limit.
For input formatting, we adopt a simple question-answer template: ``Question:[question]\verb|\n|Answer:" without any additional prompting techniques.

\subsection{Answer Pool Marginalization}
After completing the tree expansion process, we generate final answers by appending the prompt ``Therefore, the answer is'' to each branch, resulting in an answer pool $\mathcal{A} = \{A_1, A_2, ..., A_p\}$. 
To select the final answer from the pool, several straightforward approaches can be applied: (1) majority voting based on answer frequency, and (2) single-branch selection that selects the answer from individual branch with the highest score metrics (\textit{e.g.}, final score or mean score).
Instead, we propose branch-aggregation selection that determines the final answer by aggregating branch score metrics for each answer.

Specifically, for each candidate answer $A_i$, we collect its supporting branches $R(A_i)$, which consists of all branches that arrive at $A_i$ as their final answer, formally defined as $R(A_i) = \{R \mid \text{answer}(R) = A_i\}$.
Then we compute the value of each branch from its node score sequence $[S_1, S_2, ..., S_N]$, using its final score $S_N$ as the branch value.
For each unique answer, we then aggregate the values of all its supporting branches by summation: $Value(A_i) = \sum_{R \in R(A_i)} Value(R)$.
Finally, we select the answer with the highest aggregated value as our final answer: $A^* = argmax_{A_i \in \mathcal{A}} Value(A_i)$.
We provide a detailed comparative analysis of different answer selection methods in Section \ref{sec: Answer Value Calculation Analysis}.

\section{Experiments}

\textbf{Dataset and LLMs} We evaluate our method on popular mathematical reasoning benchmarks: 
(1) GSM8K \cite{cobbe2021training}, a challenging dataset of grade school math problems; 
(2) MultiArith (MA) \cite{roy2016solving}; 
(3) SVAMP \cite{patel2021nlp}; 
(4) MAWPS \cite{koncel2016mawps};
and a logical reasoning benchmark:
(5) CoinFlips (CF) \cite{srivastava2022beyond}.
For our experiments, we use the same LLMs as in Section \ref{main:Linear Probing in CoT}: Mistral-7b, Gemma2-2b, and Phi-1.5.

\textbf{Baselines}
We compare our approach with six representative baselines:
(1) Greedy Decoding: Selects the highest-probability token at each step of generation.
(2) Zero-shot CoT prompting (Zs CoT) \cite{kojima2022large}: Appends "Let's think step by step" to questions, encouraging step-wise problem-solving without task-specific training.
(3) Zero-shot Tree of Thought prompting (Zs ToT) \cite{yao2023tree}: Generate multiple reasoning steps via prompting and evaluate through self-assessment prompts.
(4) Activation-Steering (Act-S)\cite{hojerimproving}: Steers model activations along a direction vector derived from the difference between CoT and non-CoT hidden states.
(5) Chain-of-Thought Decoding (CoT-Dec) \cite{wang2024chain}: Generates multiple solution paths and selects the most confident one based on the average probability margin between the top two token predictions in the answer segment.
(6) Self-consistency (SC) \cite{wang2022self}: Employs a majority voting mechanism across multiple generated responses to identify the most consistent answer.
In appendix\ref{app:computation complexity} and \ref{app:baselines reproducing details}, we provide a detailed analysis of the time complexity of each method and reproducing details.

\textbf{Hyperparameters Config}
For the branching phase, we set the depth $m$ = 3 and beam width $n$ = 3. At each step, we generate $k$ = 10 candidates and select top-$n$ based on CoT scores, with token generation lengths $T_i$ = [1, 20, 20] for steps $i$ = 1,2,3. For the completion phase, we extend each leaf node with two steps of greedy decoding, generating 100 tokens per step.
A detailed analysis of how depth $m$ and beam width $n$ affect the framework's performance is presented in Section \ref{sec: Search Space Analysis}.

\begin{table}[h]
\resizebox{0.5\textwidth}{!}{
\begin{tabular}{ccccccc}
\hline
LLM                                       & Methods & GSM8K       & MA       & SVAMP                & MAWPS           & CF          \\ \hline
\multirow{7}{*}{\rotatebox{90}{Mistral-7b}} & Greedy  & 11.92       & 15.16            & 52.66           & 58.29      & 47.60 \\
                                          & SC      & 17.13       & 27.22            & 58.00             & 66.56      & 51.60 \\
                                          & Zs CoT  & 26.17       & 50.47            & 56.33             & 69.81      & 53.00 \\
                                          & Zs ToT  & 33.82       & 52.65            & 59.75             & 71.69      & 54.40 \\
                                          & Act-S   & 15.48       & 18.93            & 56.48             & 59.45      & 48.00 \\
                                          & CoT-Dec & 25.79       & 39.76            & 58.66             & 64.78      & 51.20 \\
                                          & Ours    &\textbf{38.18}  &\textbf{58.57} &\textbf{61.33} &\textbf{80.64}  &\textbf{56.80}  \\ \hline
                            
\multirow{7}{*}{\rotatebox{90}{Gemma2-2b}}  & Greedy  & 6.42           & 5.53           & 38.53        & 46.16         & 44.40 \\
                                          & SC      & 7.59           & 8.41           & 40.00           & 47.00         & 49.80 \\
                                          & Zs CoT  & 16.92          & 42.11          & 39.33           & 51.69         & 48.40 \\
                                          & Zs ToT  & 18.73          & 45.74          & 44.08           & 55.37          & 53.20 \\
                                          & Act-S   & 7.38          & 11.43          & 41.36           & 49.84          & 45.00 \\
                                          & CoT-Dec & 14.34          & 33.22          & 38.99           & 50.28           & 47.40 \\
                                          & Ours    & \textbf{20.62} & \textbf{50.00} & \textbf{48.66}   & \textbf{63.86} & \textbf{54.60} \\ \hline

\multirow{7}{*}{\rotatebox{90}{Phi-1.5}}    & Greedy  & 5.69             & 24.44          & 24.33            & 33.74     & 42.60 \\
                                            & SC      & 25.02            & 33.88          & 29.03            & 39.16     & 46.20 \\
                                            & Zs CoT  & 7.21             & \textbf{83.88} & 39.33            & 65.18      & 54.40 \\
                                            & Zs ToT  & 29.56             & 53.45         & 41.85           & 67.18     & 55.60 \\
                                            & Act-S   & 6.65             & 25.65          & 28.66            & 37.84      & 44.20 \\
                                            & CoT-Dec & 23.12            & 25.00          & 23.66            & 50.05         & 49.40 \\
                                            & Ours    & \textbf{37.38}   & 80.56          & \textbf{45.66}   &\textbf{68.45} & \textbf{56.80} \\ \hline
\end{tabular}
}
\vskip -0.1in
\caption{Problem solving accuracy compared with baselines across LLMs and datasets}
\label{tab: main results}
\end{table}

\begin{table}[]
  \label{answer selection}
  \resizebox{0.5\textwidth}{!}{
  \begin{tabular}{cccccc}
  
  \hline
  LLMs & Methods & GSM8K & MultiArith & SVAMP &WAMPS\\ \hline
  \multirow{5}{*}{\rotatebox{90}{Mistral-7b}} 
  & Cover Rate  & 85.44          & 91.65          & 90.33          & 94.33 \\
  & F Agg/BoN    & 38.18/27.84     & \textbf{58.57}/32.78  & \textbf{61.33}/52.45  & \textbf{80.64}/63.18  \\
  & M Agg/BoN     & 38.21/24.92        & 55.15/33.42        & 58.44/51.52       & 77.33/61.42\\
  & IR Agg/BoN   & \textbf{42.92}/ 23.52 & 57.53/35.63       & 60.21/47.42         & 79.33/64.21\\
  & Vote & 39.21 & 56.15 & 59.44 & 78.33 \\ \hline
  \multirow{5}{*}{\rotatebox{90}{Gemma2-2b}}
  & Cover Rate   & 79.65                & 84.33                 & 88.44                    & 90.74 \\
  & F Agg/BoN    & \textbf{20.62}/11.52 & 50.00/25.53         & \textbf{48.66}/16.82   & \textbf{63.86}/35.42 \\
  & M Agg/BoN    & 18.15/10.83         & 47.77/27.63         & 45.33/23.63             & 61.33/43.85\\
  & IR Agg/BoN   & 21.53/13.53        & \textbf{51.21}/34.42 & 47.44/19.42            & 62.33/40.91\\
  & Vote          & 19.21               & 48.15                & 46.44                     & 62.33            \\ \hline
  
  \multirow{5}{*}{\rotatebox{90}{Phi-1.5}}
  & Cover Rate      & 84.33                 & 89.42                     & 88.63                 & 92.82 \\
  & F Agg/BoN       & 37.38/21.72         & \textbf{80.56}/56.86    &\textbf{45.66}/29.74    & 68.45/49.72 \\
  & M Agg/BoN       & 35.77/20.44          & 77.21/49.63            & 42.33/31.42         & 65.53/50.82\\
  & IR Agg/BoN      & \textbf{38.21}/21.93 & 79.53/48.82                     & 44.65/30.84         & \textbf{69.84}/51.72\\
  & Vote            & 36.21               & 78.15                       &  43.44                & 66.49 \\ \hline
  \end{tabular}
  }
  \vskip -0.1in
  \caption{Performance comparison of different answer selection methods. F Agg/BoN, M Agg/BoN, and IR Agg/BoN represent the accuracy of branch-aggregation/best-of-N selection using final scores, mean scores, and increase ratio respectively. Vote shows the accuracy of majority voting baseline.}
  \vspace{-5mm}
  \label{tab: answer value}
  \end{table}

\subsection{Main Experimental Analysis}
As shown in Table \ref{tab: main results}, our method consistently outperforms baseline approaches in most scenarios, achieving substantial improvements in problem solving accuracy.

\begin{table*}[t]
  \centering
  \resizebox{0.92\textwidth}{!}{
  \begin{tabular}{cccccccc}
  \hline
  \multirow{2}{*}{LLMs}   & \multirow{2}{*}{Dataset}  & \multicolumn{3}{c}{LR} 
                          & \multicolumn{3}{c}{SVM} \\ \cline{3-8}   
                          &   & MLP  & Attn       & \multicolumn{1}{c|}{Hidden states}   & MLP & Attn  & Hidden states \\ \hline
  \multirow{4}{*}{Mistral-7b}    & GSM8K         & 35.21/18.42  & 34.57/17.23           & \textbf{38.18}/13.75  & 36.43/17.42 & 33.37/6.23 & 38.32/12.39\\
                                 & MultiArith    & 51.55/16.33 & 49.82/21.65          & \textbf{58.57}/23.39  & 49.55/18.33 & 50.59/18.65 & 57.45/20.11   \\
                                 & SVAMP         & 35.43/28.72 & 34.65/23.23          & \textbf{61.33}/25.76   & 36.43/27.42 & 35.43/26.42 & 60.39/17.23    \\
                                 & WAMPS         & 69.55/39.33 & 72.81/42.65          & \textbf{80.64}/52.46  & 68.55/38.92 & 77.52/7.33 & 79.38/18.65    \\ \hline
  \multirow{4}{*}{Gemma2-2b}     & GSM8K         & 15.41/6.51  & \textbf{20.62}/15.86 & 17.41/12.39     & 18.41/5.62 & 18.41/4.81 & 19.91/9.71   \\
                                 & MultiArith    & 42.41/17.74  & \textbf{50.00}/21.23 & 40.92/29.72    & 46.41/6.23 & 48.41/10.63 & 41.65/15.85   \\
                                 & SVAMP         & 35.43/13.92 & \textbf{48.66}/27.23  & 45.33/30.76   & 36.43/16.42 & 47.43/5.42 & 44.81/13.84    \\
                                 & WAMPS         & 48.55/28.74 & 63.86/29.65  & 47.64/23.59            & 49.55/27.33 & \textbf{64.55}/26.37 & 46.84/16.65    \\ \hline
  \multirow{4}{*}{Phi-1.5}       & GSM8K         & 15.41/8.69  & \textbf{37.38}/11.82 & 14.14/9.28    & 16.41/4.84 & 36.41/23.48 & 13.83/7.93    \\
  \multicolumn{1}{l}{}           & MultiArith    & 46.41/24.82 & \textbf{80.56}/21.23 & 45.41/24.61    & 47.41/15.71 & 75.41/34.72 & 44.28/19.47   \\
  \multicolumn{1}{l}{}           & SVAMP         & 34.43/16.71 & 45.66/27.75  & 43.39/25.76            & 35.43/21.42 & \textbf{46.43}/24.24 & 43.14/24.85    \\
  \multicolumn{1}{l}{}           & WAMPS         & 47.55/27.48 & \textbf{68.45}/49.65  & 46.64/32.53   & 48.55/26.33 & 67.55/45.62 & 45.72/29.38     \\ \hline
  \end{tabular}
  }
  \vskip -0.1in
  \caption{Performance comparison of different classifiers (LR and SVM) and representations (MLP, Attention, Hidden states) using accuracy scores on top-3 and bottom-3 layers (reported as top-3/bottom-3) on math reasoning datasets.}
  \label{tab: classifier feature analysis}
  \end{table*}

\textbf{Cross-Model Analysis}
Our method shows robust performance gains across different LLM scales. 
For the larger Mistral-7b model, we observe the most significant improvements, with our method achieving 38.18\% accuracy on GSM8K, surpassing the strongest baseline (Zs CoT) by 14.01\%.
The performance advantage maintains for smaller models like Gemma2-2b and Phi-1.5, where our method improves GSM8K accuracy by 3.7\% and 12.36\% respectively compared to their best baselines. 
This demonstrates our method's effectiveness and generalizability across different model scales.

\textbf{Cross-Dataset Analysis}
Our method shows varying effectiveness across different datasets.
On GSM8K's complex multi-step problems, we demonstrate consistent superiority across all models.
For MultiArith, while achieving strong performance with Mistral-7b (58.57\%) and Gemma-2-2b (50.00\%), Phi-1.5 shows slightly lower accuracy (80.56\%) compared to Zs CoT (83.88\%), suggesting simpler arithmetic problems might benefit less from our approach.
On SVAMP and MAWPS, we maintain consistent improvements, with notable gains on MAWPS (3.27\%-12.17\% over the best baseline).
On CoinFlips, we achieve 56.80\% accuracy, which is higher than the best baseline (Zs CoT) by 12.20\% in phi-1.5.
Notably, we train our classifier only on GSM8K training set and use this single classifier across all datasets, demonstrating strong generalization to various mathematical reasoning datasets.

\subsection{Answer Selection Analysis}
\label{sec: Answer Value Calculation Analysis}

Table \ref{tab: answer value} presents a comprehensive comparison of different approaches for final answer selection from the answer pool.
We first examine the coverage rate - the percentage of correct answers present in the pool - which indicates an upper bound for selection accuracy.
The high coverage rates (79\%-94\%) demonstrate that our exploration strategy effectively traverses the response space and captures valid reasoning chains.

We then evaluate two main selection paradigms: Best-of-N(BoN) selection and branch-aggregation selection, across three score sequence metrics: final scores, average scores, and increase ratio (defined as the proportion of score improvements between adjacent nodes).
Our analysis shows that branch-aggregation selection consistently outperforms BoN selection across all metrics. 
Among the three metrics, final scores yield the best performance, followed by increase ratio, while average scores show relatively inferior results.
Additionally, we benchmark these methods against the baseline majority voting approach, which shows superior performance to single-branch selection but falls short of branch-aggregation selection.

\subsection{Classifier Feature Analysis}
Table \ref{tab: classifier feature analysis} shows the performance comparison of different classifiers features, including classifier type, representation type and layers range.

\textbf{Classifier type Study}
We comparing Support Vector Machine(SVM) and LR classifiers, we observe their comparable performance across different representations, layers, and LLMs.
While SVM shows slightly better results in some cases, the differences are marginal, suggesting both classifiers can effectively guide the search process.

\textbf{Representation Layer Analysis}
We analyze the impact of layer by comparing top-3 and bottom-3 layers based on their classification F1-scores.
The results show that across all LLMs, using top-performing layers consistently outperforms bottom layers.
For GSM8K, the average improvements are 31.36\%, 29.79\%, and 30.04\% on Mistral-7b, Gemma-2-2b, and Phi-1.5 respectively, demonstrating that layer selection significantly affects search effectiveness.

\textbf{Representation Type Study}
Hidden states yield the best search performance for Mistral-7b, while attention activations prove more effective for both Gemma-2-2b and Phi-1.5. 
This pattern mirrors the relative strengths we observed in classification performance, suggesting a consistent relationship between classifier logit and reward.

\subsection{Tree Search Space Scaling Laws}

\label{sec: Search Space Analysis}
We investigate how different search space size configurations affect model performance by varying beam width $n$ and tree depth $m$.
For each configuration, we maintain the initial sampling size $k=10$ while adjusting width $n \in \{1,2,3,4,5,6\}$ and depth $m \in \{1, 2, 3, 4, 5, 6\}$. 
All generated chains are constrained to a maximum length of 240 tokens, with tokens evenly distributed across depth steps ($T_i = 240/m$ tokens per step).
Figure \ref{fig: Search Space} demonstrates how performance varies with different combinations of width and depth, using Phi-1.5 on GSM8K.

\textbf{Beam Width Impact}
The accuracy improves substantially as the beam width increases, demonstrating the benefits of maintaining more parallel branches at each expansion step. 
The improvement trend begins to plateau around width = 4, suggesting that maintaining 3-4 parallel reasoning trajectories provides sufficient exploration while remaining computationally efficient. 
Further increasing the beam width yields diminishing returns, possibly due to the introduction of more noise than CoT content.

\textbf{Search Depth Study}
The accuracy improves as search depth increases and reaches its peak at depth 3 or 4.
Beyond this optimal depth, performance gradually declines, suggesting that deeper searches may accumulate errors and explore irrelevant reasoning paths. 
This optimal depth aligns with general problem solving patterns, as most tree search methods can solve reasoning problems within 3-4 key reasoning steps \cite{he2024advancing, wang2024seed}.

\begin{figure}[ht]
  \begin{center}
  \centerline{\includegraphics[scale = 0.31]{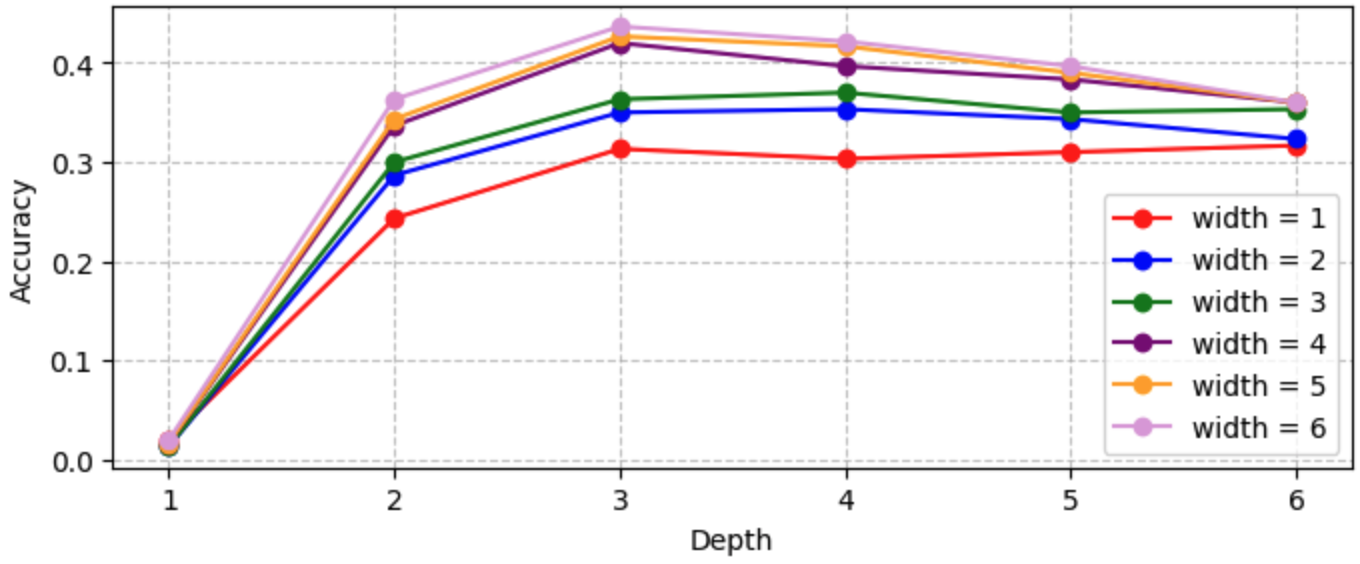}}
  \vskip -0.1in
  \caption{The accuracy plot when scaling the search space with different expansion depth and beam width.}
  \label{fig: Search Space}
  \end{center}
  \vskip -0.2in
  \end{figure}

\section{Related Work}

\subsection{Reasoning Ability Enhancement in LLMs}
Methods to improve LLMs' reasoning ability can be categorized into tuning-based and inference-time approaches.
Tuning-based methods focus on fine-tuning LLMs with high-quality rationales. STaR \cite{zelikman2022star} iteratively bootstraps rationales through generation and filtering. TRICE \cite{hoffman2024training} employs MCMC sampling to construct training data with rationales and leverages rationalization for failed cases. 
DeepseekR1 \cite{guo2025deepseek} uses outcome reward to reinforces the CoT ability.
Inference-based methods design structured reasoning frameworks to guide LLMs during inference. Chain-of-thought (CoT) \cite{wei2022chain, kojima2022large} breaks down reasoning into sequential steps. Tree-of-thoughts (ToT) \cite{yao2024tree, long2023large} enables multi-path exploration with backtracking. Graph-of-Thoughts (GoT) \cite{besta2024graph} extends to arbitrary graph topologies for complex reasoning patterns.
Tree-based methods have emerged as mainstream by balancing exploration capability with structural simplicity.

\subsection{Linear Representation Hypothesis in LLMs}
The Linear Representation Hypothesis (LRH), initially proposed in word embeddings \cite{mikolov2013linguistic}, suggests that semantic features exist as linear directions in activation space. 
Recent work has extended this to LLMs \cite{luo2024pace, von2024language, zou2310representation, park2311linear}, showing that high-level concepts like truthfulness \cite{li2024inference, burns2022discovering}, morality \cite{zou2310representation}, and factual knowledge \cite{gurnee2023language} can be represented linearly in model's activation space.
This finding enables two key applications: detection and guidance. For detection, linear classifiers can effectively probe specific concepts \cite{chen2024inside, du2024haloscope}, with their high performance indicating the linear encoding of these concepts. For guidance, these identified directions can be leveraged to steer model behavior during inference \cite{lee2409programming, li2024inference, zhao2024steering}.

\section{Conclusion}
In this work, we present ThoughtProbe, a pure inference-time framework that leverages LLMs' hidden reasoning features to improve reasoning performance. 
Our probing experiments reveal that LLM architectures encode CoT differently across representation types and layers, with simple linear classifiers achieving strong performance.
We also show theoretically and empirically that these classifiers effectively score and rank candidates to guide the search process.
Building on this discovery, we develop a classifier-guided beam search algorithm that effectively explores the reasoning space by prioritizing promising candidates.
Our framework combines tree-structured exploration with branch aggregation for final answer determination, enabling systematic utilization of valid reasoning chains within response space.
Extensive experiments across multiple benchmarks demonstrate the effectiveness of our approach, achieving significant improvements over existing methods. 

\section*{Limitations}
Despite our approach's effectiveness, we acknowledge several limitations that may warrant further investigation. 
First, our current implementation relies on fixed token lengths to segment intermediate thoughts during tree node expansion, which may disrupt natural reasoning by forcing arbitrary branching points. 
Future work should explore more flexible, semantic-aware splitting criteria to better preserve complete units of reasoning.
Second, while the answer pool achieves promising coverage rates for correct answers, our final answer selection process has room for improvement. 
The observable gap between coverage and accuracy suggests current chain evaluation and branch-aggregation strategies may not optimally capture answer quality. 
Future research could develop more sophisticated scoring metrics and aggregation methods.
Another limitation is increased inference-time cost: ThoughtProbe's tree expansion and classifier evaluation require more computation and higher latency than standard single-path or sampling methods, especially with greater search depth and beam width. This may limit practicality in latency-sensitive settings. Future work could reduce overhead via more efficient pruning, early stopping, or lightweight scoring.


\bibliography{custom}

\newpage
\appendix

\newpage
\appendix
\section{Appendix}
\label{sec:appendix}

\subsection{Justification of using classifier logit as a scoring and ranking mechanism}

In this section, we provide a brief proof that a binary classifier’s logit implies ordering equivalence with preference rewards in the Bradley-Terry model\cite{sun2024rethinking,bradley1952rank}.
\label{sec:proof}

\textbf{Binary Classification Setting:}
Let $\mathcal{X}$ be the input space. For any $x \in \mathcal{X}$:
\begin{itemize}
\item $P(y = 1|x)$ denotes the probability of positive class
\item $l(x) := \text{logit}$ is the classifier logit output
\item $P(y = 1|x)$ = sigmoid($l(x)$)
\end{itemize}

\textbf{Bradley-Terry Preference Model:}
Given a question $q$, for any two responses $x_1, x_2 \in \mathcal{X}$:
\begin{itemize}
\item $P(x_1 \succ x_2|q)$ denotes preference probability that response $x_1$ is preferred to $x_2$
\item $P(x_1 \succ x_2|q) = \frac{\exp(r(x_1|q))}{\exp(r(x_1|q)) + \exp(r(x_2|q))} = \text{softmax}(r(x_1|q), r(x_2|q))$.
\item $r(\cdot|q)$ is the underlying reward function that evaluates the quality of a response. We omit the question $q$ in the following discussion for brevity.
\end{itemize}

\begin{lemma}[Classification-Preference Connection]
\label{lemma:classification-preference-connection}
For any instance $x \in \mathcal{X}$:
\begin{equation}
P(y = 1|x) = \mathbb{E}_{j\sim p(j)}[P(x \succ j)]
\end{equation}
\end{lemma}

\begin{proof}
In binary classification, we can view the process as a competition where:
\begin{itemize}
\item $x$ competes against a random competitor $j$
\item $P(y = 1|x)$ represents the winning probability of $x$
\item When $j$ is randomly sampled from $p(j)$, this probability equals $\mathbb{E}_{j\sim p(j)}[P(x \succ j)]$
\end{itemize}
\end{proof}

Suppose that the classifier is trained on preference data derived from the Bradley-Terry model, where preference pairs are treated as binary classification data, we have the following theorem to connect the classifier logit and the reward function:

\begin{theorem}[Logit implies reward ordering]
Given two instances $x_1$ and $x_2$:
\begin{equation}
l(x_1) > l(x_2) \Rightarrow r(x_1) > r(x_2)
\end{equation}
\end{theorem}

\begin{proof}
We prove $l(x_1) > l(x_2) \Rightarrow r(x_1) > r(x_2)$ using preference probabilities' strict monotonicity.

Given $l(x_1) > l(x_2)$, the sigmoid function's strict monotonicity means $P(y = 1|x_1) > P(y = 1|x_2)$. 
Using the Classification-Preference Connection\ref{lemma:classification-preference-connection} and Bradley-Terry model, we show:

\begin{align*}
\sum_{j} p(j) \cdot &\left[\frac{\exp(r(x_1))}{\exp(r(x_1)) + \exp(r(j))}\right. \\
&\left.- \frac{\exp(r(x_2))}{\exp(r(x_2)) + \exp(r(j))}\right] > 0
\end{align*}

Since $p(j) > 0$ and $f(r, j) = \frac{\exp(r)}{\exp(r) + \exp(r(j))}$ is strictly increasing in $r$, at least one $j$ has $f(r(x_1), j) > f(r(x_2), j)$, implying $r(x_1) > r(x_2)$.

Therefore, the logit ordering implies the reward ordering.
\end{proof}

\begin{theorem}[Logit is lower bounded by reward]
There exists a constant $C$ dose not depend on $x$, such that:
\begin{equation}
l(x) \geq r(x) - C
\end{equation}
\end{theorem}
    
    \begin{proof}
    Under the Bradley-Terry model:
    \begin{equation}
    P(y = 1|x) = \mathbb{E}_j\left[\frac{\exp(r(x))}{\exp(r(x)) + \exp(r(j))}\right]
    \end{equation}
    
    By Jensen's inequality, since $f(t) = \frac{a}{a + t}$ is convex in $t$ for $a > 0$:
    \begin{equation}
    P(y = 1|x) \geq \frac{\exp(r(x))}{\exp(r(x)) + \mathbb{E}[\exp(r(j))]}
    \end{equation}
    
    Taking logit transformation:
    \begin{align*}
    l(x) &= \text{logit} P(y = 1|x) = \log\frac{P(y = 1|x)}{1 - P(y = 1|x)} \\
         & \geq r(x) - \underbrace{\log(\mathbb{E}[\exp(r(j))])}_{C}
    \end{align*}
    \end{proof}

The above theorem shows that we can use the classifier logit as a scoring and ranking mechanism for the responses during the tree search. 
As we define CoT responses are preferred to non-CoT responses and are treated as positive samples in classifier training, the above theorem implies that the logit of CoT responses are higher than that of non-CoT responses.

It's important to note that using classifiers as complete substitutes for reward models in downstream optimization scenarios requires additional theoretical constraints and considerations. 
We refer readers to the comprehensive analysis presented in \cite{sun2024rethinking}.

\section{Baselines Reproducing Details}
\label{app:baselines reproducing details}
For the activation steering method, we follow the implementation described in \cite{hojer2025improving}, calculating a control vector $v$ using the difference-in-mean approach.
Specifically, we feed all positive and negative responses to the LLM and compute the mean hidden representations for both positive and negative responses. The control vector $v$ is then derived as the difference between these two means.
During inference, we input a question to the LLM and apply the control vector $v$ to steer the hidden representations in the forward pass.
The representation types and layers selected for steering match those used in our ThoughtProbe method for each LLM. We set the steering strength parameter to 1 across all experiments.

For the zero-shot ToT, we follow the implementation described in the appendix B.1 of \cite{yao2023tree}. 
The task format prompt is \textit{"the answer is n" where n is a number}.
The standard IO prompt is \textit{'Answer the following question with {format}: {input}'}.
The thought generation prompt is \textit{Answer the following question: {input}
Make a strategy then write. Your output should be of the following format:
Strategy:
Your strategy about how to answer the question.
Answer:
Your answer to the question. It should end with {format}}.
The voting/evaluation prompt is \textit{Given an instruction and several choices,
decide which choice is most promising.
Analyze each choice in detail, then conclude in the last line
"The best choice is {s}", where s the integer id of the choice.}.

\section{Computation Complexity Notion and Discussion}
\label{app:computation complexity}
We show computational complexity notions of all methods.
\begin{table}[h]
\centering
\begin{tabular}{ll}
\hline
\textbf{Method} & \textbf{Computational Complexity} \\
\hline
Greedy & $O(T)$ - \\
& single sequence processing \\ \
ZS-CoT & $O(T+P)$ - \\
& $P$ is prompt overhead \\
SC & $O(n \times T)$ - \\
& $n$ parallel sequences \\
CoT-Dec & $O(n \times T)$ + confidence scoring \\
Act-S & $O(T + L \times H)$ -  \\
& $L$=layers number, \\
& $H$=Steering overhead \\
ZS-ToT & $O(b^d \times (C_T + C_E))$ - \\
& $d$=tree depth, \\
& $b$=thoughts number, \\ 
& $C_T$=thought prompt cost, \\
& $C_E$=evaluation prompt cost \\
ThoughtProbe & $ O(k \times n^m \times (C + 1))$ - \\
& $n$=beam width, $m$=depth, \\
& $C$=classifier cost, \\
& $k$=candidates per step \\
\hline
\end{tabular}
\caption{Computational complexity comparison of different methods.}
\label{tab:computational-complexity}
\end{table}

The computational complexity analysis in Table \ref{tab:computational-complexity} reveals several key insights about the efficiency of different reasoning methods. Traditional approaches like Greedy decoding and ZS-CoT maintain linear complexity with respect to sequence length, making them computationally efficient but limited in reasoning capability. 
Self-consistency methods (SC) scale linearly with the number of parallel sequences, offering a reasonable balance between computational cost and performance. 

Activation steering (Act-S) introduces additional overhead proportional to the number of layers being steered, but maintains the same asymptotic complexity as standard decoding. 
In contrast, tree-based methods like ZS-ToT face exponential complexity growth ($O(b^d)$) as tree depth increases, severely limiting their practical application to complex reasoning tasks despite their strong performance.

Our ThoughtProbe method achieves a favorable complexity profile of $O(k \times n^m \times (C + 1))$, where $n$ is the beam width, $m$ is the exploration depth, $k$ is the number of candidates per step, and $C$ represents the classifier cost. While this still involves exponential growth with depth, our approach is more efficient in practice than traditional tree-based methods like ZS-ToT because: (1) we typically use smaller beam widths and depths, (2) our classifier-guided pruning effectively reduces the search space, and (3) the linear classifier overhead is minimal compared to token generation costs. This analysis demonstrates that ThoughtProbe offers an effective balance between computational efficiency and reasoning capability, making it practical for deployment in real-world applications while maintaining comparable or superior performance to more computationally intensive methods.

\newpage
\section{More Probing Experiments Details}
\label{sec:more probing results}

\subsection{Classifier Training Data Construction Details}
\label{app: classifier training data}
We first employ Top-K-Start Greedy Decoding, which explores alternative top-k tokens at the first decoding step, followed by greedy decoding for subsequent steps, to sample 10 distinct responses.
To ensure response quality, we first filter the sampled responses to remove potential repetitions of input questions as LLMs occasionally exhibit a pattern where they restate the original question after providing their answer.
We apply a post-processing step to extract only the solution-relevant content, ensuring each response contains purely reasoning and answer components without redundant question restatements.
After filtering, we prompt GPT4o as a judge to label the filtered responses as either CoT or non-CoT responses.
The prompt is as follows:

\textit{You are an expert at analyzing reasoning patterns in AI responses. Given a question-response pair, your task is to determine whether the response follows a Chain of Thought (CoT) reasoning pattern or not.}

\textit{Chain of Thought (CoT) responses show explicit step-by-step reasoning before arriving at the final answer. They break down the problem, work through intermediate steps, and show the logical progression that leads to the conclusion.}

\textit{Non-CoT responses provide direct answers without showing the reasoning process or intermediate steps.}

\textit{For the question-response pair provided below, analyze whether the response uses Chain of Thought reasoning by checking if it:
1. Shows explicit reasoning steps
2. Breaks down the problem into parts
3. Works through intermediate calculations or logical steps
4. Explains the thinking process before giving the final answer}

\textit{Reply with:
- "COT" if the response demonstrates Chain of Thought reasoning with clear intermediate steps
- "NON-COT" if the response gives a direct answer without showing the reasoning process
Question: [Question will be inserted here]
Response: [Response will be inserted here]}
Our CoT/non-CoT dataset, constructed using Mistral-7b sampling and GPT-4o annotation, directly transfers to train effective classifiers for Phi-1.5 and Gemma2-2b without reconstruction. 
The fundamental CoT vs. non-CoT distinction captures abstract reasoning patterns that generalize across LLM architectures. The rationale behind dataset reusability is that we only need paired data capturing fundamental CoT vs non-CoT reasoning patterns, which transfer across LLMs due to shared training objectives and representational structures,


\subsection{Classifier Training Settings}
For our experiments, we collected a dataset comprising 1245 positive (CoT) and 1868 negative (non-CoT) samples, with each sample representing a question-response pair. 
Analysis of response lengths revealed that CoT responses averaged 131.7 tokens, while non-CoT responses averaged only 15.8 tokens.

We extracted token-wise hidden representations from all layers of the LLM network and trained a separate classifier for each layer. To address the imbalance in response lengths between positive and negative samples, we implemented a strategic sampling approach: extracting hidden representations every five tokens for CoT responses and every token for non-CoT responses, thereby creating a more balanced training dataset.

We train the classifier using the Logistic Regression and Support Vector Machine (SVM) classifiers. 
The epoch number is 100, the learning rate is 0.001, using stochastic gradient descent (SGD) as the optimizer.

\section{More Probing Experiments Results}
\subsection{Classifier Classification Performance}
\label{app:classifier analysis}
Building upon our main findings, we conduct an extensive evaluation using both Logistic Regression (LR) and Support Vector Machine (SVM) classifiers, assessed through Accuracy, F1-score, and AUC-ROC metrics. 
As shown in Figures \ref{app:LR} and \ref{app:SVM}, these complementary metrics reinforce and extend our key observations:

(1) Representation type analysis:
The distinct patterns observed across different LLMs are consistently reflected across all metrics. In Mistral-7b, hidden states maintain their superior performance across both classifiers and all evaluation metrics, with MLP and attention activations showing comparable but slightly lower performance. For Phi-1.5, the notable fluctuation in hidden states and the stable superiority of attention outputs are robustly captured by all metrics. In Gemma2-2b, attention activations consistently demonstrate the strongest discriminative power across all evaluation criteria, while hidden states and MLP outputs show substantial variations. This consistent pattern across different evaluation frameworks strengthens our observation about model-specific architectural strategies for encoding thoughtful reasoning.

(2) Layer-wise analysis:
The layer-specific trends identified in our main analysis persist across different classification approaches and metrics. Mistral-7b's progressive improvement in deeper layers is consistently observed in both LR and SVM results, regardless of the evaluation metric used. 
The more distributed patterns in Phi-1.5 and Gemma2-2b, characterized by fluctuations without clear directional trends, are similarly preserved across all evaluation frameworks. 
These consistent findings across multiple metrics provide strong validation for our conclusions about how different LLMs architecturally encode thoughtfulness features.

\subsection{Classifier logit value analysis}
\label{app:score analysis}
We conduct a detailed analysis of classifier logit value distributions across multiple language models (Mistral-7b, Gemma2-2b, Phi-1.5). 
Using both Logistic Regression and SVM classifiers, we compare the distributional patterns between the most discriminative layers (top-3 F1 scores) and least discriminative layers (bottom-3 F1 scores).

Figures \ref{app:score_gemma_short}, \ref{app:score_gemma_long} present comprehensive comparisons of classifier logit values for Gemma2-2B. Figures \ref{app:score_mist_short}, \ref{app:score_mist_long} present comprehensive comparisons of classifier logit values for Mistral-7B. Figures \ref{app:score_phi_short}, \ref{app:score_phi_long} present comprehensive comparisons of classifier logit values for Phi-1.5. 

Specifically, Figure \ref{app:score_gemma_short} compares thoughtful correct responses against intuitive responses, while Figure \ref{app:score_gemma_long} contrasts thoughtful correct responses with thoughtful incorrect responses. 
These comparisons are conducted within both top-3 and bottom-3 performing layers (ranked by F1-scores), spanning across different classifier architectures (Logistic Regression and SVM) and various representation types (attention activations, MLP activations, and hidden states), providing a thorough validation of the scoring
and ranking ability of the classifier's logit.

Notably, in attention activations, which achieve the best classification performance among all representation types, thoughtful correct responses consistently receive higher logit values than both intuitive responses and thoughtful but incorrect responses, demonstrating the robust discriminative ability of our approach. However, this clear ranking pattern is occasionally violated in MLP activations and hidden states, where thoughtful correct responses sometimes receive lower logit values than the other response types. Moreover, this ranking trend is more pronounced in top-3 performing layers compared to bottom-3 layers, suggesting that layers with stronger discriminative power better preserve the desired response quality ordering.

Similar patterns are observed in Mistral-7B and Phi-1.5 models, indicating that our trained classifiers demonstrate strong scoring and ranking capabilities across different model architectures, making them reliable guides for thought space exploration.

\begin{figure*}[ht]
\vskip 0.2in
\begin{center}
\centerline{\includegraphics[scale = 0.45]{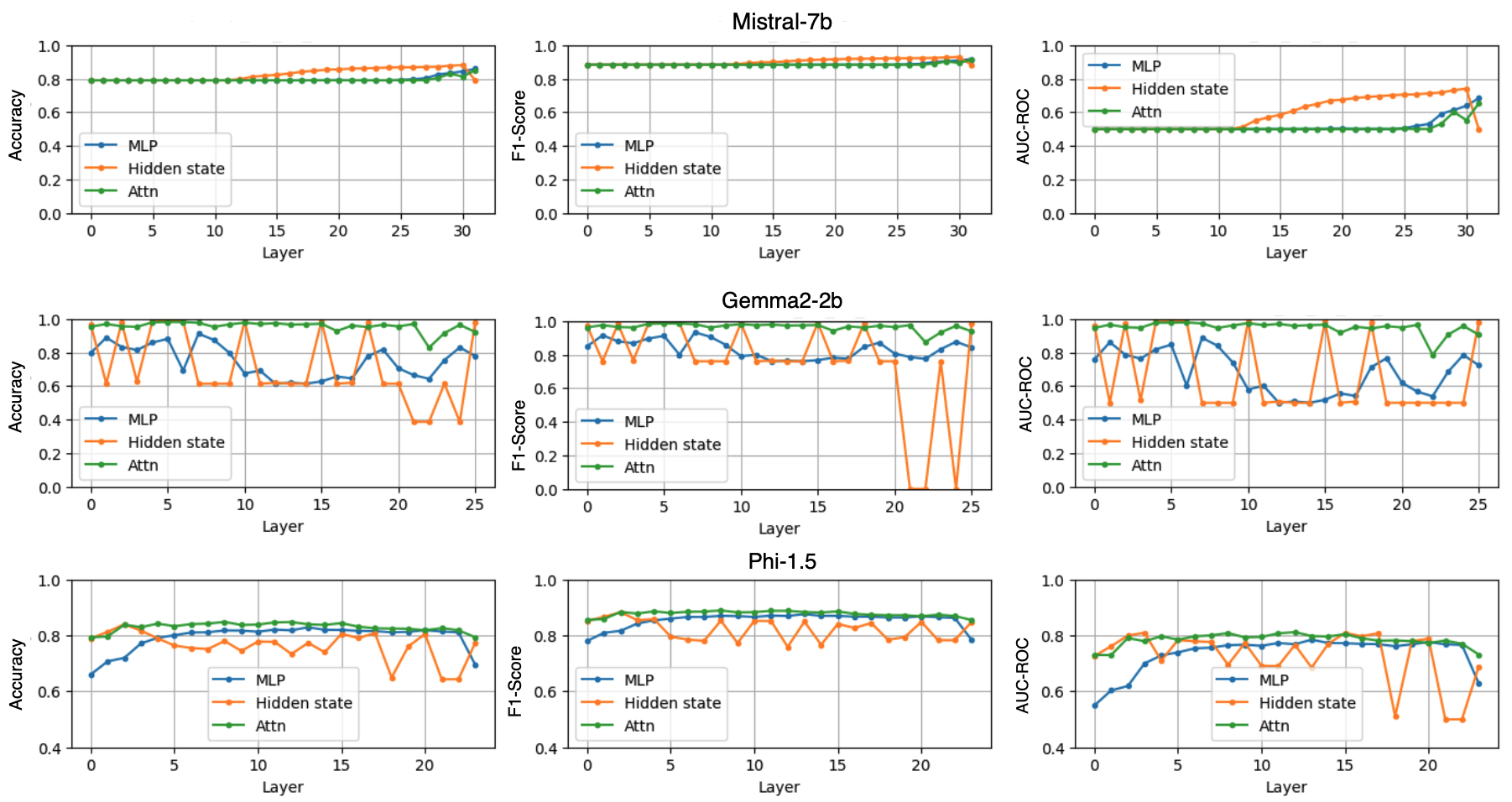}}
\vskip -0.1in
\caption{LR classification performance across LLMs and representation types}
\label{app:LR}
\end{center}
\vskip -0.2in
\end{figure*}

\begin{figure*}[ht]
\vskip 0.2in
\begin{center}
\centerline{\includegraphics[scale = 0.45]{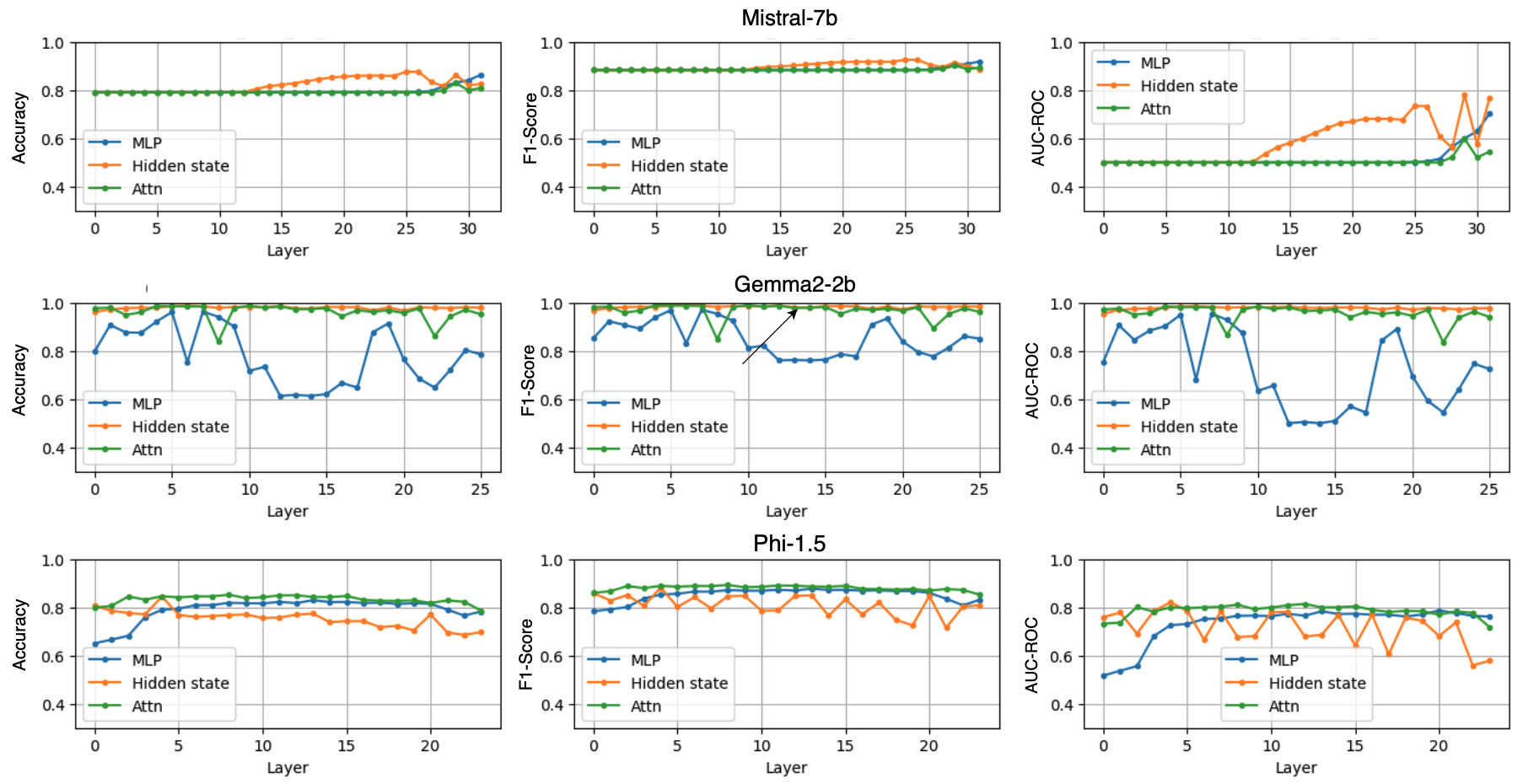}}
\vskip -0.1in
\caption{SVM classification performance across LLMs and representation types}
\label{app:SVM}
\end{center}
\vskip -0.2in
\end{figure*}

\begin{figure*}[ht]
\begin{center}
\centerline{\includegraphics[scale = 0.68]{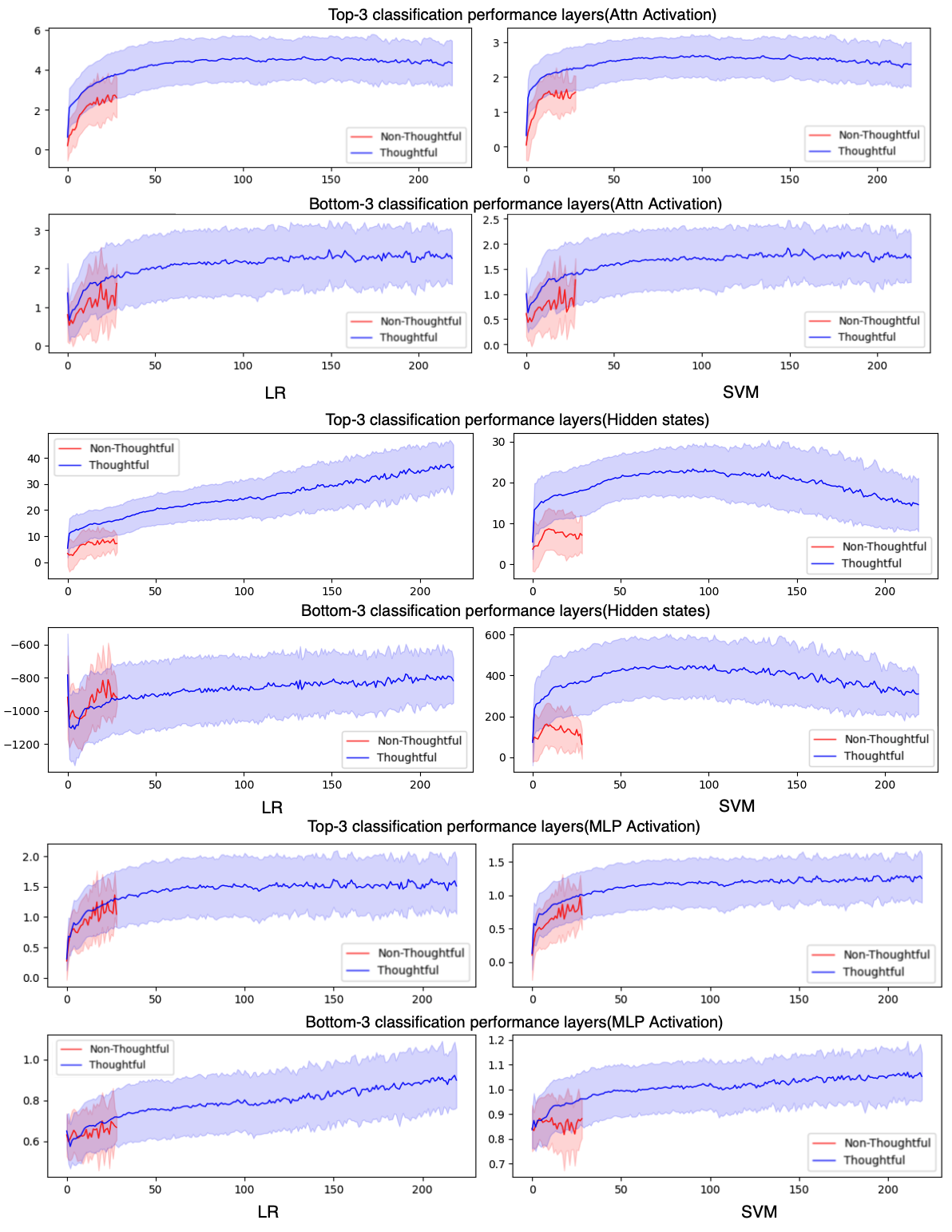}}
\vskip -0.1in
\caption{Mean logit values and variance regions in Gemma2-2b, comparing lengthy thoughtful correct responses with concise intuitive incorrect ones.}
\label{app:score_gemma_short}
\end{center}
\vskip -0.2in
\end{figure*}

\begin{figure*}[ht]
\begin{center}
\centerline{\includegraphics[scale = 0.67]{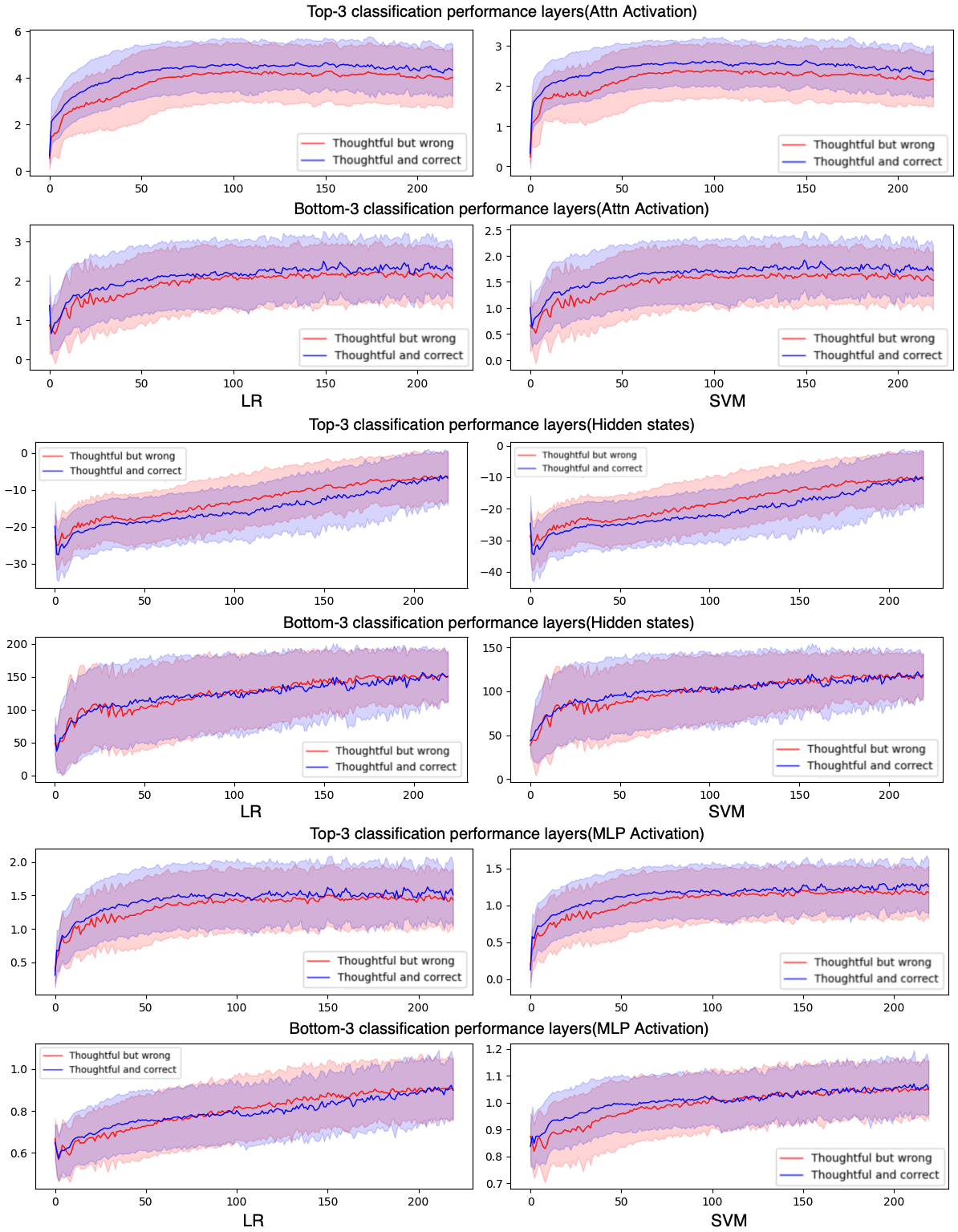}}
\vskip -0.1in
\caption{Mean logit values and variance regions in Gemma2-2b, comparing lengthy thoughtful correct responses with lengthy incorrect ones.}
\label{app:score_gemma_long}
\end{center}
\vskip -0.2in
\end{figure*}

\begin{figure*}[ht]
\begin{center}
\centerline{\includegraphics[scale = 0.70]{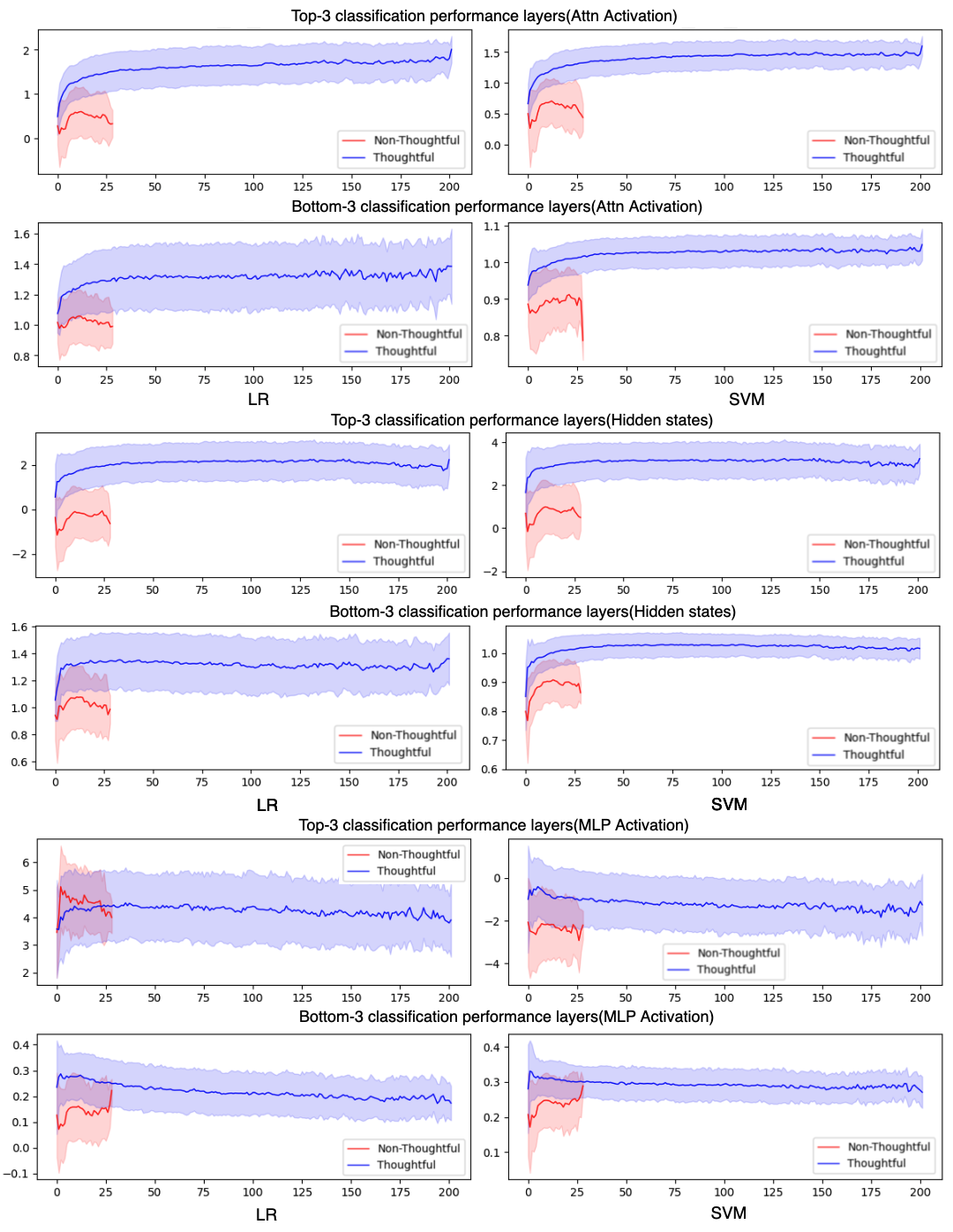}}
\vskip -0.1in
\caption{Mean logit values and variance regions in Mistral-7b, comparing lengthy thoughtful correct responses with concise intuitive incorrect ones.}
\label{app:score_mist_short}
\end{center}
\vskip -0.2in
\end{figure*}

\begin{figure*}[ht]
\begin{center}
\centerline{\includegraphics[scale = 0.67]{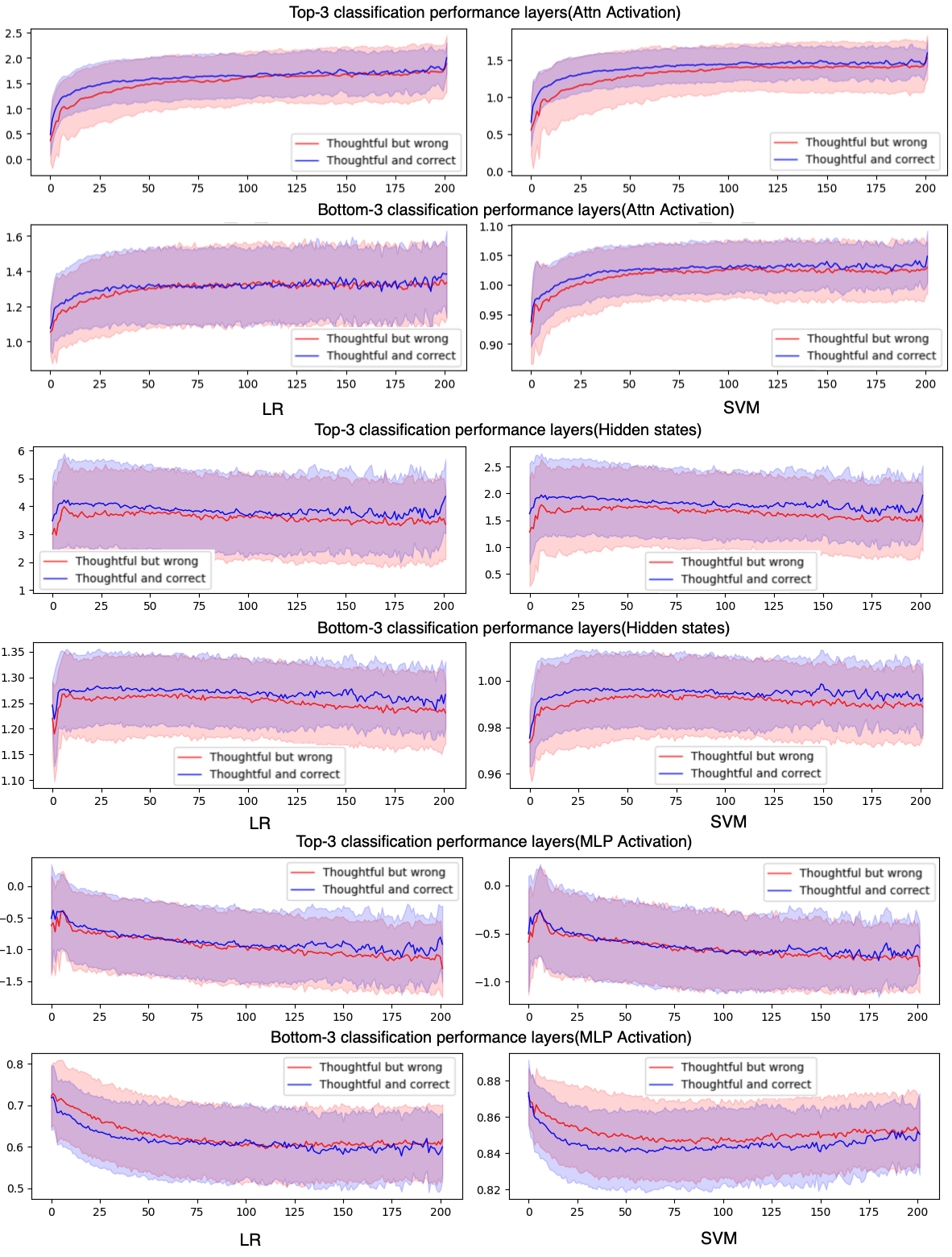}}
\vskip -0.1in
\caption{Mean logit values and variance regions in Mistral-7b, comparing lengthy thoughtful correct responses with lengthy incorrect ones.}
\label{app:score_mist_long}
\end{center}
\vskip -0.2in
\end{figure*}

\begin{figure*}[ht]
\begin{center}
\centerline{\includegraphics[scale = 0.6]{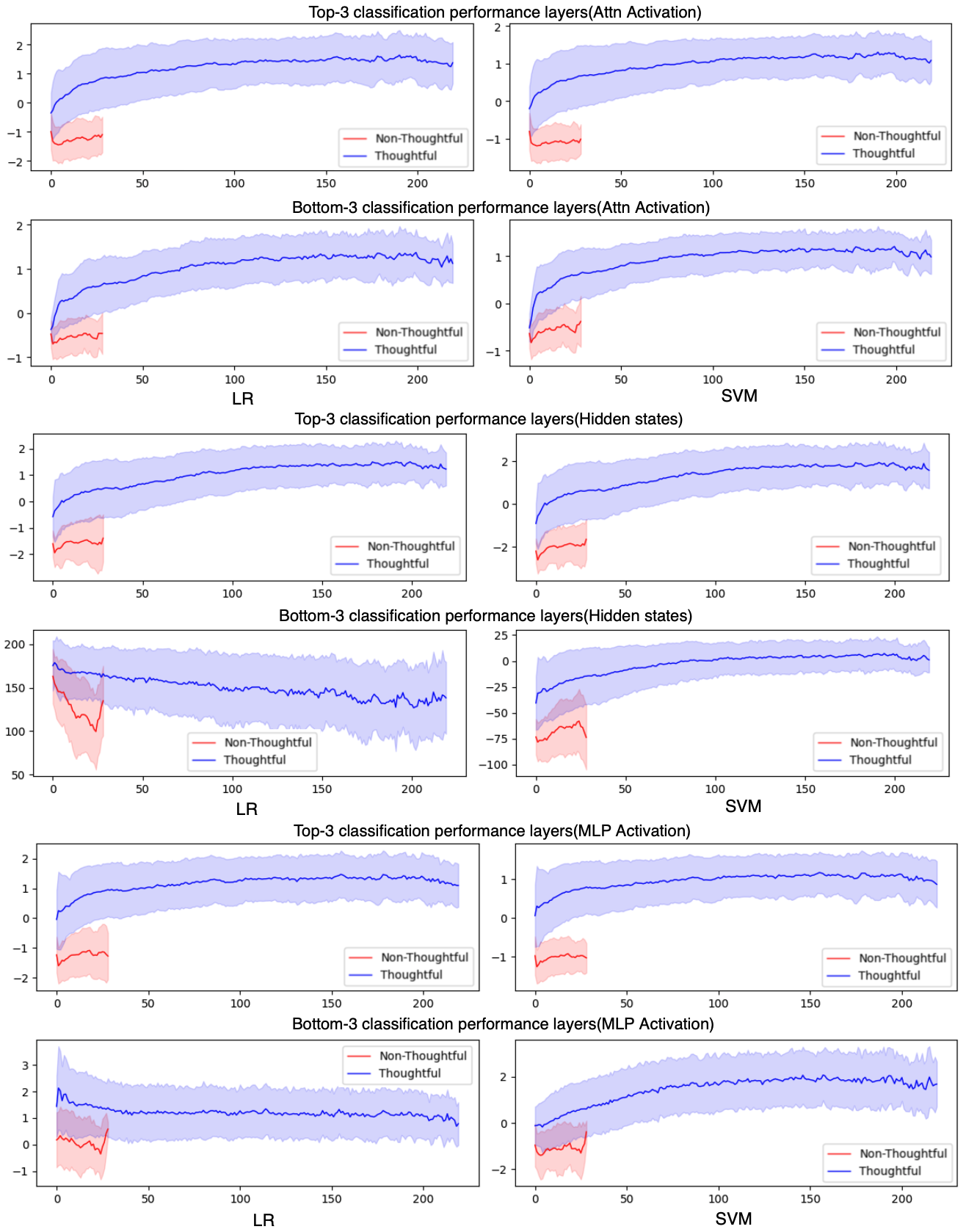}}
\vskip -0.1in
\caption{Mean logit values and variance regions in Phi-1.5, comparing lengthy thoughtful correct responses with concise intuitive incorrect ones.}
\label{app:score_phi_short}
\end{center}
\vskip -0.2in
\end{figure*}

\begin{figure*}[ht]
\begin{center}
\centerline{\includegraphics[scale = 0.67]{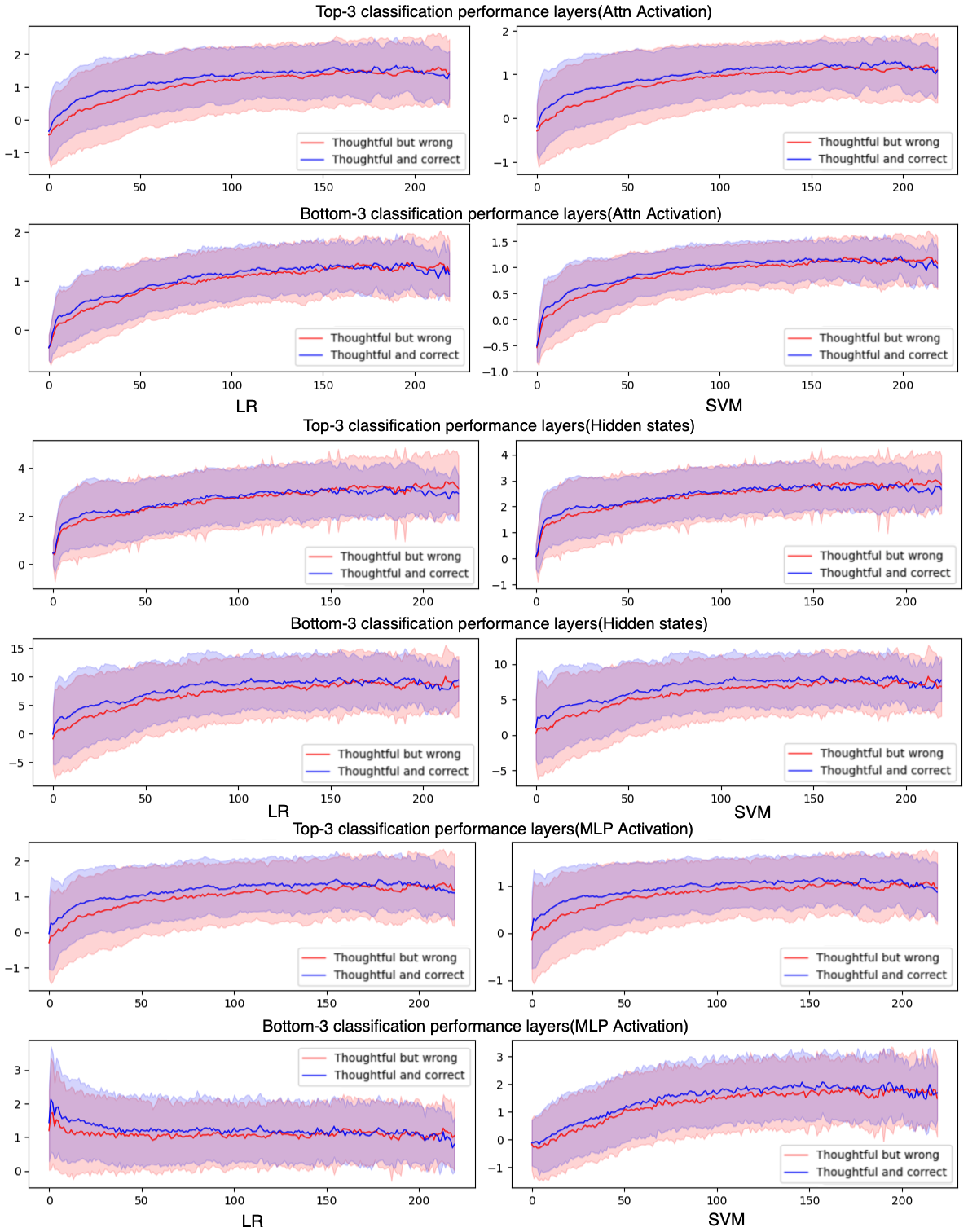}}
\vskip -0.1in
\caption{Mean logit values and variance regions in Phi-1.5, comparing lengthy thoughtful correct responses with lengthy incorrect ones.}
\label{app:score_phi_long}
\end{center}
\vskip -0.2in
\end{figure*}

\end{document}